%% file: main.tex
\documentclass[letterpaper, 10 pt, conference]{ieeeconf}  % Comment this line out if you need a4paper

\IEEEoverridecommandlockouts                              % This command is only needed if 
\usepackage{graphics} % for pdf, bitmapped graphics files
\usepackage{epsfig}
\usepackage{amsmath} % assumes amsmath package installed
\usepackage{amssymb}  % assumes amsmath package installed
\usepackage{bm}
\usepackage{cite}
\usepackage{color}
\usepackage{balance}
\usepackage{caption,subcaption}
\usepackage{siunitx}
\usepackage{gensymb}

\usepackage{amsthm}
\theoremstyle{definition}
\newtheorem{definition}{Definition}[section]
\newtheorem{theorem}{Theorem}[section]

\newtheorem{lemma}[theorem]{Lemma}

\title{\LARGE \bf Be your own Benchmark: \\
No-Reference Trajectory Metric on Registered Point Clouds}

\author{Anastasiia Kornilova$^{1}$ Gonzalo Ferrer$^{1}$% <-this % stops a space
\thanks{The authors are with Skolkovo Institute of Science and Technology.
 {\tt\small \{anastasiia.kornilova,g.ferrer\}@skoltech.ru}% <-this % stops a space
\newline 978-1-6654-1213-1/21/\$31.00 \textcopyright 2021 IEEE}
}

\begin{document}

\maketitle
\thispagestyle{empty}
\pagestyle{empty}

%%%%%%%%%%%%%%%%%%%%%%%%%%%%%%%%%%%%%%%%%%%%%%%%%%%%%%%%%%%%%%%%%%%%%%%%%%%%%%%%
\begin{abstract}
This paper addresses the problem of assessing trajectory quality in conditions when no ground truth poses are available or when their accuracy is not enough for the specific task~--- for example, small-scale mapping in outdoor scenes. In our work, we propose a no-reference metric, Mutually Orthogonal Metric (MOM), that estimates the quality of the map from registered point clouds via the trajectory poses. MOM strongly correlates with full-reference trajectory metric Relative Pose Error, making it a trajectory benchmarking tool on setups where 3D sensing technologies are employed. We provide a mathematical foundation for such correlation and confirm it statistically in synthetic environments.
Furthermore, since our metric uses a subset of points from mutually orthogonal surfaces, we provide an algorithm for the extraction of such subset and evaluate its performance in synthetic CARLA environment and on KITTI dataset. The code of the proposed metric is publicly available as pip-package.
% For evaluation, data collection setup should provide data from 3D sensing technologies, widespread in modern robotics systems.
\end{abstract}

%%%%%%%%%%%%%%%%%%%%%%%%%%%%%%%%%%%%%%%%%%%%%%%%%%%%%%%%%%%%%%%%%%%%%%%%%%%%%%%%

\input{parts/1_intro}
\input{parts/2_related}
\input{parts/3_background}
\input{parts/4_method}
\input{parts/5_experiments}
\input{parts/6_conclusion}

\bibliographystyle{bib/IEEEtran}
\balance
\bibliography{bib/IEEEabrv,bib/main}

\end{document}

%% file: parts/1_intro.tex
\section{Introduction}

%  What is the problem?
% Problem 1: we need to find a proper pose estimation algorithm for our specific task and tune its' hyper-parameters to perfectly fit our needs

In the past years, the explosion of 3D technologies has spanned numerous research works and algorithms on odometry and SLAM, creating a demand for effective and universal ways to assess the quality of the trajectories produced by those algorithms. Unfortunately, the variety of algorithm use-cases and conditions is not always covered by well-known and standard benchmarks~\cite{geiger2012we, schubert2018tum, sturm2012benchmark} and requires hyper-parameters tuning, evaluation, and comparison of algorithms on the collected target dataset.

High accurate trajectories are also important in the mapping domain~--- using them, one can aggregate small-scale maps from a set of registered sequential point clouds ~\cite{ferrer2019,liu2021balm} by using their estimated trajectory. For example, in the case of sparse point clouds from LiDAR, this allows obtaining a denser representation of the environment, increasing the level of scene understanding, such as feature extraction~\cite{li2019usip, yew2018featnet}, segmentation~\cite{behley2019semantickitti}.
% Probaly this is implicit before, kind of, but not directly mentioned
%Recent works on registered sequential PCs ~\cite{ferrer2019,liu2021balm} support this trend as well.
%The interest in this problem is supported by active development of methods that maintain the global consistency of trajectories~\cite{ferrer2019,liu2021balm}.

% by ~\AK{EF perfect place 2 along with bundle adjustment things on lidars}\GF{maybe better here? anyway we decided not to cite or refer to all mapping algos} or sequential PC alignment \cite{ferrer2019}.

\begin{figure}[h!]
    \centering
    \includegraphics[width=0.99\columnwidth]{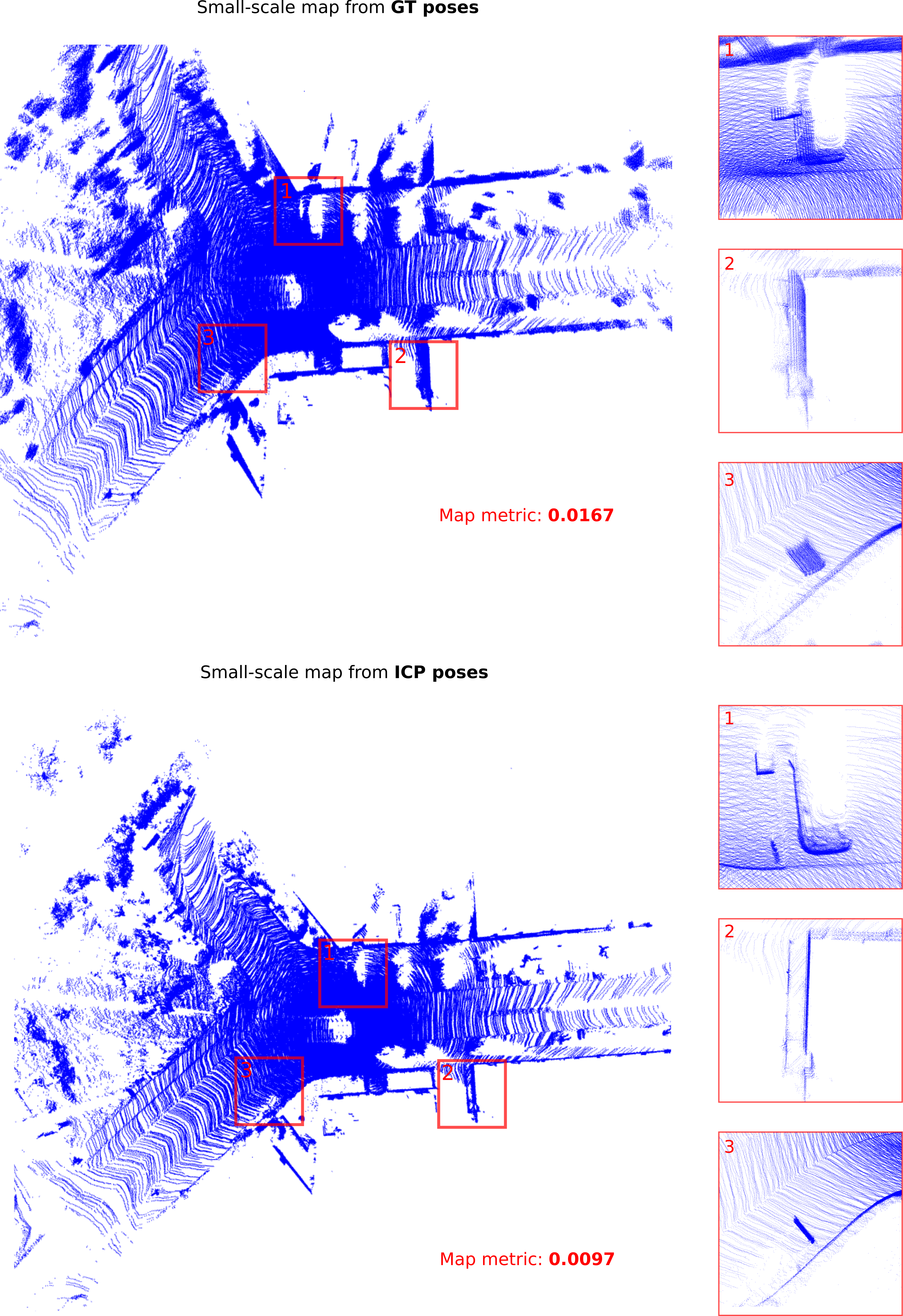}
    \caption{Demonstration of applying map metric for small-scale map aggregated from point clouds: top image presents map and its' details aggregated from GT poses, down image~--- map aggregated using registered poses from ICP algorithm. Inconsistency in the map from GT poses illustrates issues of using GNSS/INS for small-scale trajectories.} 
    % Our map metric values, on the contrary, give a more objective estimation of map quality and hence trajectory quality.
    \label{fig:teaser}
\end{figure}

The most straightforward approach to estimate quality is to use \textit{full-reference metrics} that need ground truth (GT) information and additional sensors to collect it: GT trajectory, obtained from GNSS/INS in outdoors~\cite{geiger2012we} or Motion Capture system indoors~\cite{schubert2018tum, sturm2012benchmark}, or GT map obtained from high-resolution static LiDARs~\cite{knapitsch2017tanks}. In addition to the high cost, those sensors also have their limitations: GNSS/INS is effective only on large-scale trajectories, either requires differential GPS, Motion Capture system (MoCap) is limited in space, static LiDARs could be ineffective in dynamic environments. Such examples of limitations could be found in a widespread KITTI odometry benchmark~\cite{geiger2012we} as demonstrated in Fig.~\ref{fig:teaser}~--- map aggregated from LiDAR point clouds using GT poses much more distorted than one produced by standard point2plane Iterative Closest Point (ICP) algorithm, which means imperfections of GNSS/INS GT in small-scale trajectories for outdoors.

An alternative way to overcome these limitations is to apply a \textit{no-reference metric} where the quality of the registered map and hence trajectory is estimated without GT by measuring the level of map ``inconsistency'' and ``noise''. Existing quantitative metrics for aggregated maps~~\cite{droeschel2014local, razlaw2015evaluation, zhang2021metric} tackle this problem from a heuristic point of view without demonstrating a strong correlation of map-based metrics with trajectory-based metrics on  different scenes and trajectories.

In this paper, we propose a no-reference trajectory metric MOM (Mutually Orthogonal Metric) that evaluates the quality of the trajectory via evaluating the map consistency from registered point clouds. MOM uses plane variance statistics on points from mutually orthogonal surfaces and hence strongly correlates with Relative Pose Error~\cite{kummerle2009measuring}~--- common full-reference trajectory metric to measure local trajectory consistency. Thus, MOM provides an opportunity to benchmark trajectory quality without GT information for any positioning system with 3D sensing technologies available.

We provide a mathematical interpretation to demonstrate the foundation of MOM correlation with full-reference trajectory metric and present statistical experiments with LiDAR point clouds in the CARLA simulator~\cite{Dosovitskiy17carla} on various scenes and trajectories. Because our method requires extraction of points from mutually orthogonal surfaces, we also propose an algorithm for such pre-processing and evaluate its effectiveness on real-life data from KITTI dataset scenes~\cite{geiger2012we}. Fig.~\ref{fig:teaser} illustrates that proposed map metric gives more objective values of aggregated map quality on KITTI dataset. The code of metrics, considered in this paper, is publicly available as pip-package~\footnote{\url{https://github.com/MobileRoboticsSkoltech/map-metrics}}.

The contributions of this paper are as follows:

\begin{itemize}
    \item no-reference trajectory metric for registered point clouds that gives strong correlation with Relative Pose Error and requires 40-60 times fewer points of the map in comparison to previous works;
    \item algorithm for extraction of points from mutually orthogonal subsets in point clouds;
    \item pipeline for evaluation different map metrics. 
\end{itemize}

%% file: parts/2_related.tex
\section{Related work}

The robotics, graphics, computer vision, and many other research communities have studied techniques to evaluate the quality of odometry and SLAM algorithms.
The most representative ones are: metrics for trajectory quality, metrics for point cloud quality, and no-reference metrics lying at the intersection of the first two.

% \GF{it should not be a review, only relevant references, simplify the below paragraph to half.}

Trajectory-based metrics are the main tool for evaluating odometry and SLAM algorithms in robotics. This approach requires ground truth trajectories and it can be mainly divided into two concepts: metrics based on Absolute Pose Error (APE)~\cite{lu1997globally, bar2004estimation}~--- to evaluate global trajectory consistency, and metrics based on Relative Pose Error~\cite{burgard2009comparison, kummerle2009measuring}~--- to evaluate local trajectory consistency. Among well-known odometry benchmarks the next metrics are used: the TUM dataset~\cite{sturm2012benchmark} and TUM VI~\cite{schubert2018tum}~--- APE and RPE on translation parts of trajectories, KITTI~\cite{geiger2012we}~--- a modified version of RPE that considers relative pose error with respect to the amount of drift for rotation and translation parts, EuRoC~\cite{burri2016euroc}~--- APE error.

% \GF{REMOVE: APE serves as a metric to estimate global trajectory consistency. It uses the relative difference between corresponding pairs of ground truth trajectory and one, calculated by the algorithm. Then, statistics over measures of rotation and translation parts of that difference can be estimated. In such a way, RMSE}\GF{fix to link: For instance, } statistics over translation norm is used in TUM RGB-D benchmark~\cite{sturm2012benchmark} and TUM VI benchmark~\cite{schubert2018tum}.\GF{APE is usually used in Euroc, maybe do not include the same two DS here for both...}
% \GF{ why is this here? Optionally, before calculation of APE, trajectory alignment algorithms~\cite{umeyama1991least, horn1987closed} could be applied.} 
% RPE is aimed to measure local trajectory consistency~--- to evaluate trajectory quality, it uses \GF{relative? delta} pose difference.
% \GF{Hence, in comparison to APE, error in one part of the trajectory doesn't affect many other parts of trajectories. is required?}
% Statistics based on RPE are also used in TUM RGB-D benchmark~\cite{sturm2012benchmark} and TUM VI benchmark~\cite{schubert2018tum}. In 2012, authors of KITTI dataset~\cite{geiger2012we} proposed a modified version of RPE that considers relative pose error with respect to the amount of drift for rotation and translation parts.

Point cloud-based metrics have been actively developing in the last ten years by the graphics community. The main purpose of such metrics is to evaluate algorithm quality in the tasks of dense point cloud compression (down-sampling) and denoising. To evaluate metric quality, authors use collected benchmarks~\cite{alexiou2018benchmarking} from the subjective estimation of different point clouds' quality and estimate the correlation of proposed metrics using the Pearson correlation coefficient~\cite{pearson1895notes}. Among these approaches, the next metrics could be highlighted: point2point metric, point2plane metric~\cite{tian2017geometric}, angular similarity~\cite{alexiou2018point}, projection-based methods~\cite{torlig2018novel, alexiou2018point2d} or 3D-SSIM variants \cite{meynet2019pc, alexiou2020towards}.

Map-based metrics serve as a tool for estimation trajectory by evaluating the quality of the map from registered point clouds by odometry or SLAM algorithm. Because this method does not require a ground truth trajectory or map, it could be called no-reference. Mean Map Entropy~(MME) and Mean Plane Variance~(MPV)~\cite{droeschel2014local, razlaw2015evaluation} could be included in such class, providing statistics on average entropy and planarity for every point in the map and demonstrating their correlation in indoor with trajectory-based metrics on some examples. Our method belongs to this category.

%% file: parts/3_background.tex
\section{Background}
\subsection{Relative Pose Error}
% \AK{Not finished, formal definition of different RPE types should be defined}
% \GF{Do you want to define poses or leave this open?}

To describe 3D poses of the trajectory, elements of the Special Euclidean group $SE(3)$ are used:

\begin{equation}
\begin{split}
\label{eq:se3}
    SE(3) = \left\{ T =  \begin{bmatrix} \bm{R} & \bm{t} \\ 0  & 1 \end{bmatrix} \, | \, \bm{R} \in SO(3) \, , \, \bm{t} \in \mathbb{R}^3 \right\} \\ 
    SO(3) = \left\{\bm{R} \in \mathbb{R}^{3 \times 3} | \, \bm{R}\bm{R}^T = I, \det (\bm{R}) = 1 \right\},
\end{split}
\end{equation}
where $\bm{R}$ is a rotation matrix, an element of the group of rotations $SO(3)$ and $\bm{t} = transl(T)$~--- translation.

Let $T^{gt} = \{T^{gt}_1, .., T^{gt}_N\}$ be the set of ground truth poses and $T^{est} = \{T^{est}_1, .., T^{est}_N\}$~--- poses estimated by some algorithm. Relative pose difference $E_{ij}$ is defined as:

\begin{equation}
    E_{ij} = \Delta T^{gt}_{ij} (\Delta T^{est}_{ij})^{-1} 
\end{equation}

\begin{equation}
    \Delta T_{ij} = T_{i} T_{j}^{-1}.
\end{equation}

Then, different measures of $E_{ij}$ could be considered for relative pose difference~--- for translation and rotation part. We will refer to the translation part since it presents the most important challenge for odometry and SLAM algorithm. 

\begin{equation}
    \|E_{ij}\|_{transl} = \|transl(E_{ij})\|
\end{equation}

Translation-based RPE will be considered for both 1D and 3D trajectories and will have the next forms. For 1D:

% \begin{equation}
%     T^{gt} = \{t^{gt}_1,..,t^{gt}_N | t^{gt}_i \in \mathbb{R} \}
% \end{equation}

% \begin{equation}
%     T^{est} = \{t^{est}_1,..,t^{est}_N | t^{est}_i \in \mathbb{R} \}
% \end{equation}

\begin{gather}
    \Delta T^{gt}_{ij}, \Delta T^{est}_{ij} \in \mathbb{R} \nonumber \\
    \|E_{ij}\|_{transl} = |\Delta T^{gt}_{ij} - \Delta T^{est}_{ij}|
\end{gather}

For 3D:

\begin{gather}
    \Delta T^{gt}_{ij}, \Delta T^{est}_{ij} \in SE(3) \nonumber \\
    \|E_{ij}\|_{transl} = \|transl(\Delta T^{gt}_{ij} (\Delta T^{est}_{ij})^{-1})\|_2
\end{gather}

% \begin{equation}
%     T^{gt} = \{T^{gt}_1,..,T^{gt}_N | T^{gt}_i \in SE(3) \}
%     T^{est} = \{T^{est}_1,..,T^{est}_N | T^{est}_i \in SE(3) \}
% \end{equation}

Finally, mean squared among all pairs of $i$ and $j$ is calculated:

\begin{equation}
    \textit{RPE}(T) = \frac{1}{N}\sum_{ij} \|E_{ij}\|_{transl}^2
\end{equation}

\subsection{No-reference metrics}
\label{sec:back}
This section describes the main {\em no-reference} metrics Mean Map Entropy (MME) and Mean Plane Variance (MPV)~\cite{droeschel2014local, razlaw2015evaluation} and provides a discussion of their effectiveness and boundary cases.

Mean Map Entropy $H(P)$ of aggregated map $P = \{p_1,..,p_N |p_i \in \mathbb{R}^3\}$ is based on the entropy $h(p_k)$ of every point $p_k$ in some vicinity $W(p_k)$ and is averaged among all points:

\begin{equation}
    h(p_k) = \frac{1}{2}\det \big(2\pi e \Sigma(W(p_k)) \big)
\end{equation}

\begin{equation}
    H(P) = \frac{1}{|P|}\sum_{k=1}^{|P|} h(q_k),
\end{equation}
where $\Sigma(W(p_k))$ is the sample covariance of map points in a local vicinity $W(p_k)$ of $p_k$. 

The next expansion of $h(p_k)$ illustrates that MME captures the ``spread'' of point vicinity from the map along different axis, where $\lambda_1, \lambda_2, \lambda_3$~--- eigenvalues of $\Sigma(W(p_k))$. Assuming a locally flat surface around $p_k$, then the two major eigenvalues are of equal magnitude, proportional to the radius $R$ of the vicinity and hence $h(p_k)$ will depend only on $\lambda_{min}$:

\begin{align}
    h(p_k) &= \frac{1}{2} \det 2\pi e \Sigma(W(p_k)) = \frac{1}{2} (2\pi e)^3 \det \Sigma(W(p_k)) \nonumber \\ &= \frac{1}{2} (2\pi e)^3 \lambda_1 \lambda_2 \lambda_3 \approx \frac{1}{2} (2\pi e)^3 R^2 \lambda_{min}.
\end{align}

Mean Plane Variance $V(P)$ assumes that the majority of points lie on the planar surface, then map distortions could be measured as plane errors. Vicinity $W(p_k)$ of every point $p_k$ in the aggregated map $P$ is fitted to a plane and the corresponding plane error $v(p_k)$ is calculated and averaged for all map points. Since plane error is equal to minimum eigenvalue of sample covariance of $W(p_k)$, MPV will also depend only on $\lambda_{min}$.

\begin{equation}
    V(P) = \frac{1}{|P|}\sum_{k=1}^{|P|}v(p_k) = \frac{1}{|P|}\sum_{k=1}^{|P|} \lambda_{min}.
\end{equation}

Despite authors of~\cite{razlaw2015evaluation} demonstrated correlation of this metric with trajectory errors on indoor scenes with ground truth trajectory captured by Motion Capture system, the scene setup depicted in Fig.~\ref{fig:mme-counter} demonstrates the disadvantages of MME and MPV metrics in terms of correlation with trajectory error. The magnitude of shifts (red arrow) between a pair of point clouds is the same for top and down images, statistics collected from obtained maps will be different~--- the top map will have less map-metric value because of less deviation along the right plane. Furthermore, if points are not balanced among the planes, the impact of the right not-orthogonal plane will be even more.

\begin{figure}[ht]
    \centering
    \includegraphics[width=0.65\columnwidth{}]{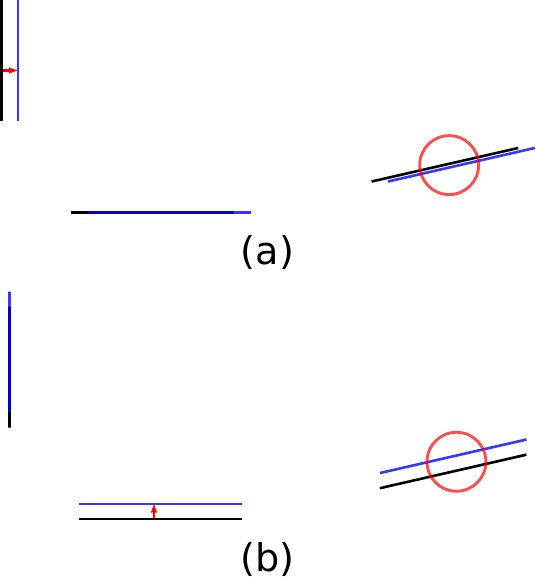}
    \caption{Example of different map metric values for the same magnitude of shift $\|E_{ij}\|$ between pair of point clouds: (a) map with less deviation among the right plane, (b) map with higher deviation among the right plane. Deviation among two other planes is the same.}
    \label{fig:mme-counter}
\end{figure}

Statistically, this could be demonstrated with a synthetic environment in 3D, which contains three planes. For such a setup, we consider 200 perturbed trajectories, whose length is five poses. Then MME and MPV metrics are calculated for maps aggregated from perturbed trajectories as well as trajectory error is calculated between GT and perturbed trajectories. Finally, statistics for the following three configurations are collected: (i) three orthogonal planes with an equal number of points on every plane, (ii) three orthogonal planes with a different number of points on every plane, (iii) three non-orthogonal planes with an equal number of points on every plane. The collected statistics and corresponding correlation coefficients (Pearson~\cite{pearson1895notes}, Spearman~\cite{spearman1961proof}, and Kendall~\cite{kendall1938new}) are depicted in Fig.~\ref{fig:mme-counter-stat} and numerically demonstrates degradation of correlation when points are not-balanced among planes and when planes are non-orthogonal.

\begin{figure*}[ht]
    \centering
    \includegraphics[width=0.9\textwidth]{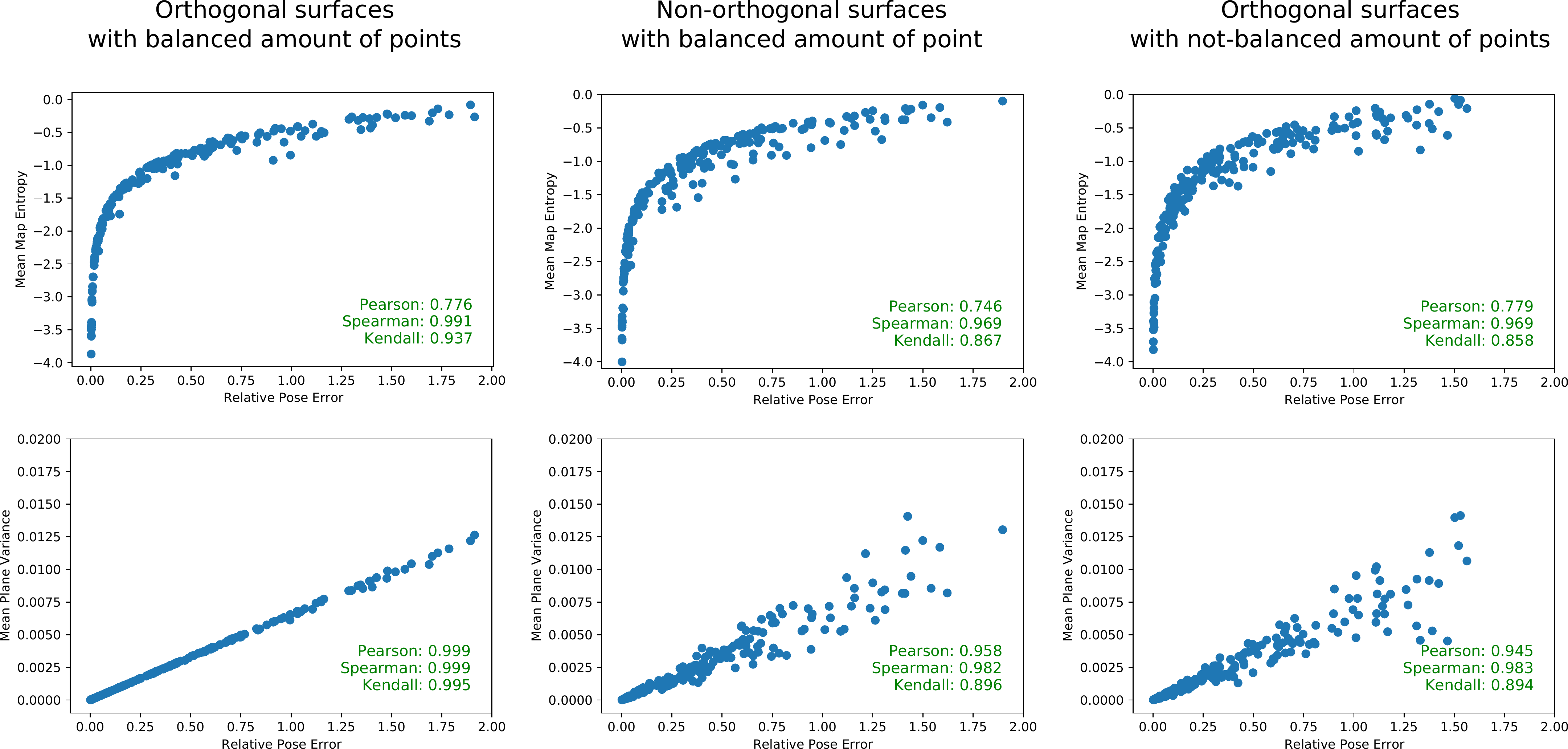}
    \caption{Results of statistical experiment that demonstrates degradation of map metric correlation with full-reference trajectory error in case of non-orthogonality of planes and not-balanced points among them.}
    \label{fig:mme-counter-stat}    
\end{figure*}

%% file: parts/4_method.tex
\section{Map-metric based on orthogonality}
% \GF{first paragraph explaining the motivation maybe it is not needed. Just start explaining the orthogonal case.}
% \GF{move to later most of it}
% The main idea of our method relies on the observations that statistics on points from mutually orthogonal balanced planes have a good correlation with trajectory-based error. This section provides proofs of properties that demonstrate such correlation in the ideal conditions that makes metrics correlation more interpretable. Also, we present the algorithm for extraction of such mutually orthogonal subsets from point clouds.

From the example depicted in Fig.~\ref{fig:mme-counter} it could be noticed that considering only the first two mutually orthogonal planes removes such ambiguity in map metric because distortions in one direction don't affect distortion in others. The issue with not-balanced even orthogonal surfaces could be solved by considering every mutually orthogonal direction separately and averaging statistics along every direction. This section demonstrates how this intuition on using map statistics from mutually orthogonal surfaces could be extended to any number of poses with some assumptions and proposes an algorithm for extracting points from mutually orthogonal surfaces from the point cloud.

\subsection{No-reference metric based on orthogonality for translation shifts}
\label{sec:lemma}
% \GF{yes, it feels a little abrupt. Maybe we can go as this: 1D relation RPE and MPV, then for 3D RPE for aligned frames is just the sum of RPE coordinates and then explain that we can simply define an orthogonal basis to decompose the RPE, in the proof you can say that any rotation of the coordinates results in a separation of RPE's and the relate each one to MPV< for 1D, looking here for max. Let's give it a try}

Firstly, let us consider the next lemma which connects map-metric and trajectory metric (RPE).

\begin{lemma}
\label{lemma:corr}
For a set of one dimensional points $A = \{a_1, .., a_N | a_i \in \mathbb{R}\}$, their sample covariance $\Sigma(A)$ linearly depends on relative pose error $RPE(A)$ without dependence on points themselves, only on their amount.

\begin{equation}
    \frac{\Sigma(A)}{RPE(A)} = \alpha(N)
\end{equation}
\end{lemma}

\begin{proof}
Expanding $\Sigma(A)$:

\begin{align}
\Sigma(A) = \frac{1}{N-1}\sum_{i}(a_i - \frac{1}{N}\sum_j a_j)(a_i - \frac{1}{N}\sum_{j} a_j)^T =  \nonumber \\
\frac{1}{N^2(N-1)} \sum_{i} \big(Na_i - \sum_j a_j\big)^2 =  \nonumber \\
\frac{1}{N^2(N-1)} \sum_{i} (N^2a_i^2 - 2Na_i\sum_{j} a_j + (\sum_{j} a_j)^2) =  \nonumber \\ 
\frac{1}{N^2(N-1)} \Big((N^2 - N) \sum_{i} a_i^2 - 2N\sum_{i,j, i \neq j}a_ia_j\Big) = \nonumber \\
\frac{1}{N(N-1)} \Big((N - 1) \sum_{i} a_i^2 - 2\sum_{i,j, i \neq j}a_ia_j\Big).
\end{align}

Expanding $\textit{RPE}(A)$:

\begin{equation}
\begin{split}
    \textit{RPE}(A) = \frac{1}{N}\sum_{i \neq j} (a_i-a_j)^2 \\ = \frac{1}{N} ((2N - 2)\sum_{i} a_i^2 - 4\sum_{i\neq j} a_i a_j) = \\
    \frac{2}{N}((N - 1)\sum_{i} a_i^2 - 2\sum_{i, j, i\neq j} a_i a_j).
\end{split}
\end{equation}

Finally, the relation between those values does not depend on value of points from the original set $A$:

\begin{equation}
\begin{split}
    \frac{\Sigma(A)}{\textit{RPE}(A)} = \frac{1}{2(N - 1)}.
\end{split}
\end{equation}

\end{proof}

Let us consider $N$ shifted planes along direction $n$ as depicted in the Fig.~\ref{fig:orth-along-axis} with coordinates $A = \{a_1, .., a_N | a_i \in \mathbb{R}\}$ with respect to $n$. In case when points are balanced among those planes, plane variance, which is actually $\lambda_{min}$ of covariance matrix in some vicinity, will be equal to $\Sigma(A)$, and following lemma~--- equal to RPE along axis $n$.

\begin{figure}[h!]
    \centering
    \includegraphics[width=0.8\columnwidth]{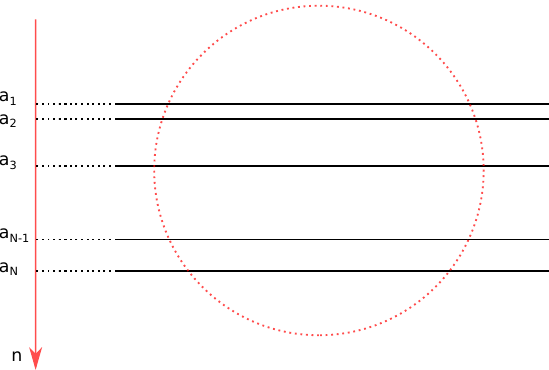}
    \caption{Example of N planes shifted along axes $n$ to demonstrate connection between Relative Pose Error for 1D and plane variance for set of points on those shifted planes.}
    \label{fig:orth-along-axis}
\end{figure}

Then, connection between RPE for 1D trajectory to RPE for 3D trajectory could be demonstrated.

\begin{lemma}
    Relative Pose Error for 3D trajectory could be decomposed in sum of three 1D Relative Pose Errors along each axis of any orthogonal basis, with the assumption that poses are a set of translations without rotations $T = \{T_1,.., T_N | T_i \in SE(3), R_i = \mathbf{I}_3 \}$.
\end{lemma}

\begin{proof}
The next expansion demonstrates decomposition of RPE for 3D poses to sum of RPE for 1D along coordinate axis.

\begin{align}
    \textit{RPE}(T) = \frac{1}{N}\sum_{ij}^n \| trans(T_i \cdot T_j^{-1}) \|^2 = \nonumber \\
    \frac{1}{N}\sum_{ij} \| t_{i} - t_{j} \|^2 = \nonumber \\ 
    \frac{1}{N}\sum_{ij}\frac{1}{3} ((t^0_{i}- t^0_{j})^2 + (t^1_{i}-t^1_{j})^2 + (t^2_{i} - t^2_{j})^2) \nonumber \\ 
    = \frac{\textit{RPE}(t^0) + \textit{RPE}(t^1) + \textit{RPE}(t^2)}{3},
\end{align}
where $t^0_i, t^1_i, t^2_i$ define coordinates of translation vector $t_i$, and $t^0 = \{t^0_1,.., t^0_N | t_i^0 \in \mathbb{R}\}$, $t^1 = \{t^1_1,.., t^1_N | t_i^1 \in \mathbb{R}\}$, $t^2 = \{t^2_1,.., t^2_N | t_i^2 \in \mathbb{R}\}$~--- set of 1D poses along every coordinate axis.

Rotation of coordinate system gives the same decomposition but along new coordinate axis, which means that any orthogonal basis could be chosen to measure RPE in terms of 1D decomposition.
\end{proof}

Finally, connecting decomposition of 3D RPE in any orthogonal basis in 1D RPE (Lemma 2) and dependency between 1D RPE and plane variance (Lemma 1), map metric \textbf{MOM} (Mutually Orthogonal Metric) with interpretable correlation with trajectory metric could be defined in the next form.
\begin{enumerate}
    \item Extract points that lie on mutually orthogonal surfaces, directions of those surfaces will present orthogonal basis for RPE decomposition.
    \item Calculate mean plane variance of the map for every subset of points associated with one basis direction.
    \item Average mean plane variance along three directions of orthogonal basis. 
\end{enumerate}

Because any orthogonal basis could be chosen with respect to Lemma 2, it is desirable to choose such basis that will cover as maximum as possible elements of the map to gather more statistics.

\subsection{Orthogonal subset extraction from point clouds}
Following the discussion on statistics based on points from orthogonal surfaces, the problem is reduced to finding a maximum subset of points that lie on the orthogonal surfaces. Let $\mathcal{N} = \{n_i \, | \, n_i \in \mathbb{R}^3, \|n_i\| = 1\}$ be a set of point cloud normals. The set we are interested in can be defined as a set where every pair of normals are either co-linear or orthogonal. Because all sensor measurements tend to produce noisy observations and hence noisy normals, we will consider a concept of nearly co-linear and nearly orthogonal vectors in some vicinity $\epsilon$ defined in the following way.

\begin{definition}[Nearly co-linear]
Let pair of unit vectors $n_i$ and $n_j$ be nearly co-linear in some vicinity $\epsilon$ if
\begin{equation}
|n_i \cdot n_j| < \epsilon.
\end{equation}
\label{def:colinear}
\end{definition}

\begin{definition}[Nearly orthogonal]
\label{def:orth}
Let pair of unit vectors $n_i$ and $n_j$ be nearly orthogonal in some vicinity $\epsilon$ if 
\begin{equation}
|n_i \cdot n_j| > 1 - \epsilon.    
\end{equation}
\end{definition}

To gather more statistics from orthogonal surfaces, we are interested in the maximum subset of $\mathcal{N}$ where every pair of normals follows the property of nearly co-linearity or nearly orthogonality. This task could be formulated as a max clique problem on a graph $G(\mathcal{N})$ where normals are graph vertices, and every nearly co-linear or nearly orthogonal pair of vertices is connected via edge. Max clique problem is an NP-complete task which means that solving it for standard point cloud ($>$10k points) is not effective. To provide an effective solution, we do Agglomerative clustering of normals into nearly located normals with complete linkage criteria and then apply max clique search algorithm for the centers of obtained clusters ($<$100). Because several max cliques of clusters could be obtained, the max clique with the maximum amount of points in clique clusters is chosen to gather more statistics. Example on Fig.~\ref{fig:kitti-orth} demonstrates such orthogonal subset for KITTI point cloud extracted using the described algorithm.

\begin{figure}[h!]
    \centering
    \includegraphics[width=0.9\columnwidth]{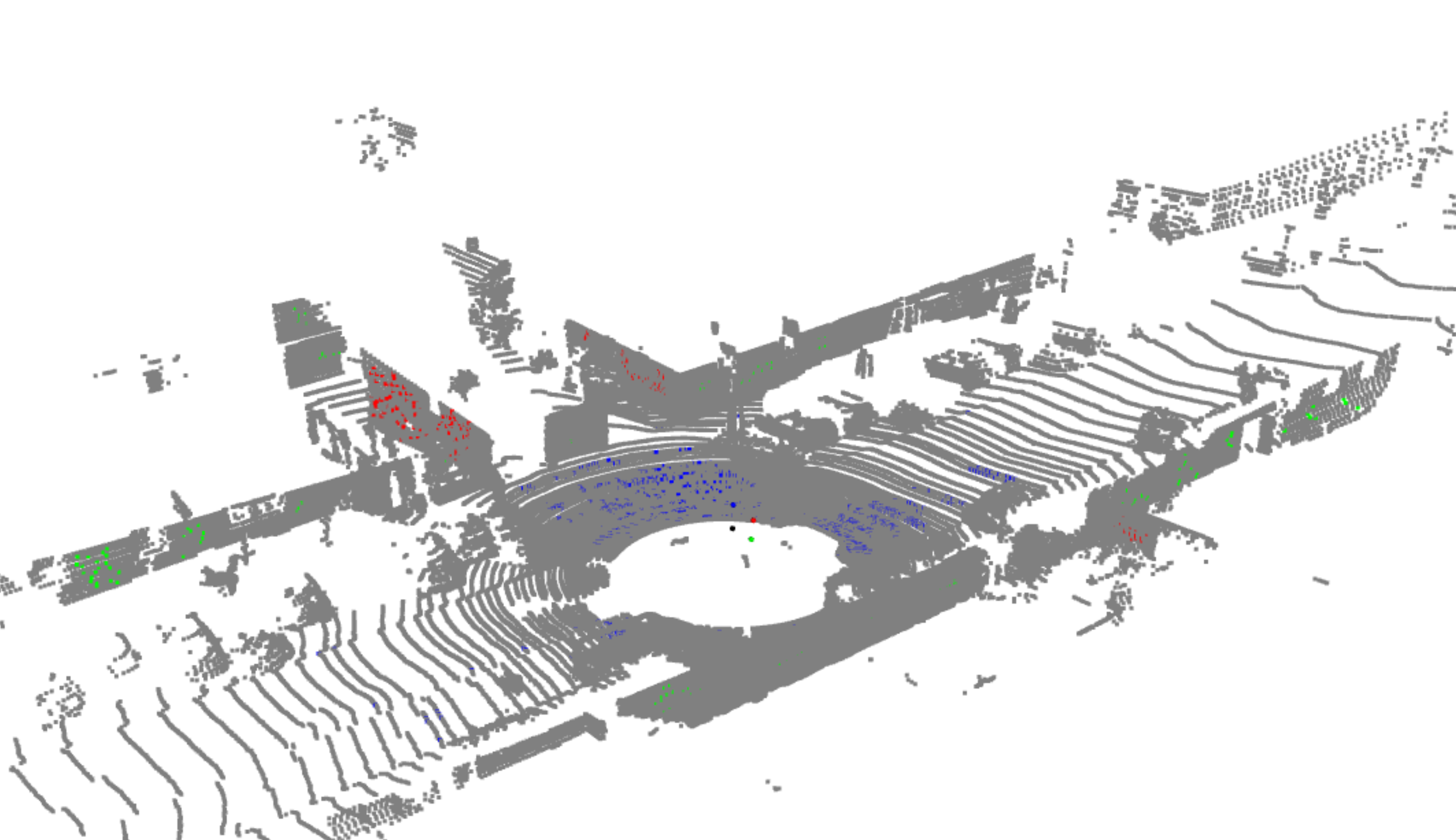}
    \caption{Example of orthogonal subset extraction for KITTI. Grey-colored points are the whole KITTI point cloud; red, green, and blue depict points related to mutually orthogonal surfaces.}
    \label{fig:kitti-orth}
\end{figure}

\begin{figure*}[h]
    \centering
    \includegraphics[width=0.9\textwidth]{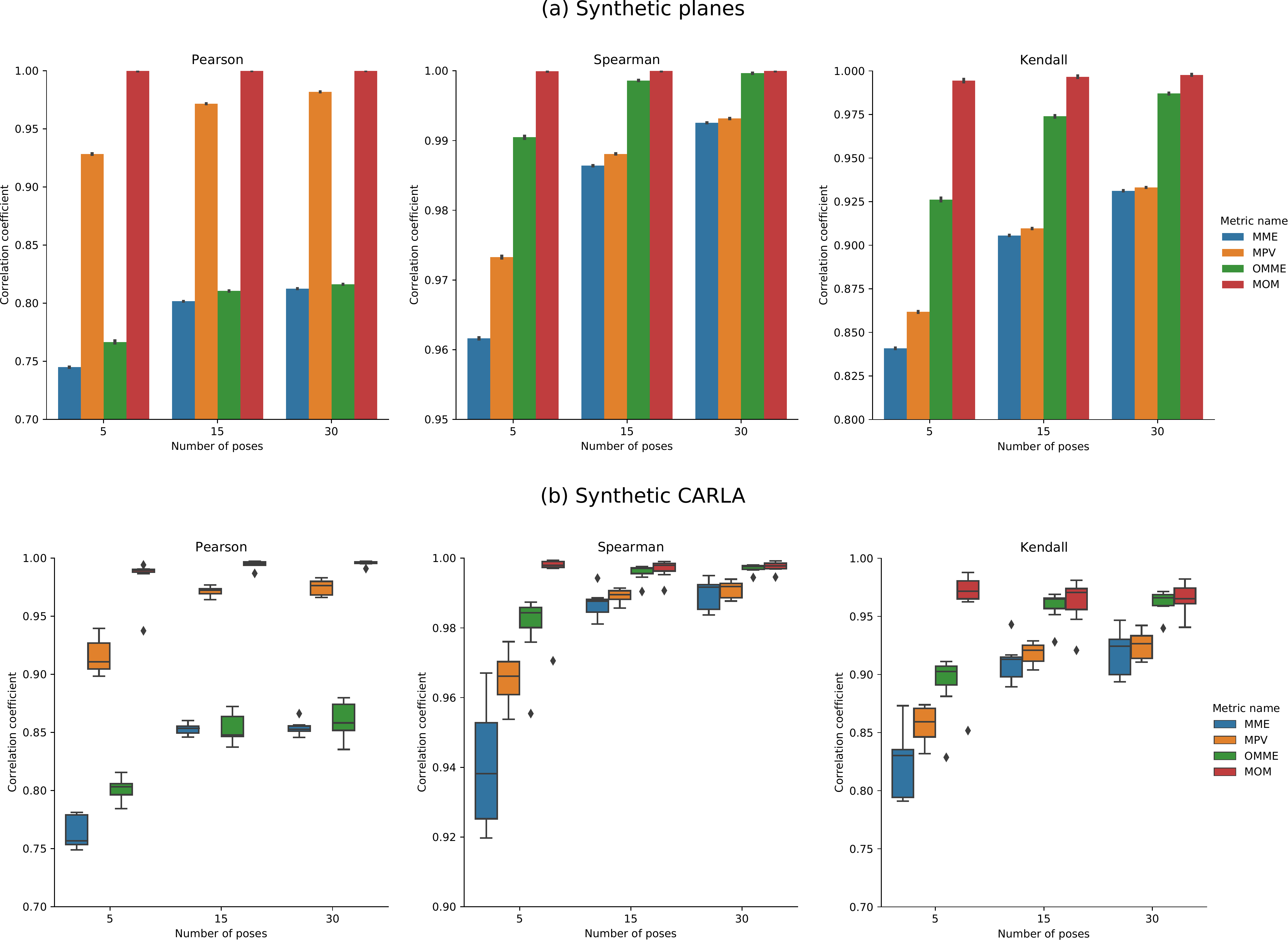}
    \caption{Metric evaluation in synthetic environments: (a)~--- synthetic planes environment, (b)~--- CARLA simulator environment. First, 200 perturbed trajectories of lengths 5, 15, 30 are sampled, and correlation with metrics is considered using Pearson, Spearman, and Kendall coefficients.}
    \label{fig:synth-eval}
\end{figure*}

\subsection{MOM on Degenerate cases}

For correlation with trajectory error, our metric requires to have at least three mutually orthogonal surfaces on the scene. Usually, this condition is satisfied in many environments: indoors (apartments, offices), outdoors (urban environments). In case when three mutually orthogonal surfaces are not available, statistics on the available amount of directions could be used -- it will still give a measure of map inconsistency that could be used to evaluate the algorithm, but without strong correlation with the trajectory error.

%% file: parts/5_experiments.tex
\section{Experiments}
The experiments section is divided into two parts: evaluating metrics quality in synthetic environments with ground truth poses available and evaluation algorithm for orthogonal subset extraction on real-life data (KITTI dataset).

\subsection{Synthetic environment}
Synthetic environments allow obtaining ground truth information about point cloud poses along the trajectory and hence estimate the quality of correlation with trajectory error. In this part of the work, we consider two types of synthetic environments: (i) environment with randomly generated configurations of synthetic planes, mainly because discussed metrics rely on this property of environment, (ii) outdoor simulator CARLA~\cite{Dosovitskiy17carla} with generated LiDAR point clouds for different city maps and trajectories.

To evaluate metric quality, the next pipeline is used:

\begin{enumerate}
    \item a sequence of point clouds and corresponding ground truth poses $T_{gt}$ are generated;
    \item ground truth trajectory is perturbed to obtain $T_{est}$: for plane environment perturbation of only translation part is used to follow lemma conditions, for CARLA scenes both rotation and translation parts are perturbed;
    \item the full-reference metric (RPE) is calculated for the pair of trajectories $(T_{gt}, T_{est})$;
    \item the no-reference metric is calculated for the map aggregated from point clouds with perturbed trajectory;
    \item after sampling 200 perturbed trajectories, a linear correlation and monotonic correlation between full-reference and no-reference metric are calculated using Pearson~\cite{pearson1895notes}, Spearman~\cite{spearman1961proof}, and Kendall~\cite{kendall1938new} statistics.
\end{enumerate}

The next metrics were considered for evaluation:

\begin{itemize}
    \item MME~--- original Mean Map Entropy~\cite{razlaw2015evaluation};
    \item MPV~--- original Mean Plane Variance~\cite{razlaw2015evaluation};
    \item OMME~--- Mean Map Entropy on subset of mutually orthogonal surfaces (ours), following discussion on it dependency from plane variance in Sec.~\ref{sec:back};
    \item MOM~--- Mean Plane Variance on subset of mutually orthogonal surfaces (ours).
\end{itemize}

\subsubsection{Synthetic planes}

Because the main assumption made in the metric proposals is that environments mainly contain planar surfaces, firstly, we consider the synthetic environment that contains only planes. The purpose of this is to evaluate the stability of considered metrics on different sets of planes, taking into account the problem of non-orthogonal cases and unbalanced number of points.

To generate different configurations of such an environment, the next process was proposed. First, we generate normals for planes: they contain three fixed orthogonal normals, count of other planes are sampled uniformly from $[0, 7]$, normals for those additional planes are sampled uniformly by the uniform sampling of $\theta$ from $[0, 2\pi)$ and $\phi$ from $[0, \pi)$ and obtaining normal from the formula:

\begin{equation}
    n = [ \sin \phi \cos \theta \ \ \sin \phi \sin \theta \ \ \cos \phi ].
\end{equation}

Secondly, displacement from the origin is generated for every plane in the range $[-10, 10]$ for every coordinate axis avoiding cases of plane intersections. Also, the density of every plane is sampled uniformly from the range $[30, 100]$ per \SI{1}{\metre\squared}.

This configuration provides orientation and translation of all planes from the origin as well as their density. For every pose, we uniformly sample points on those planes in radius \SI{1}{\metre\squared}. 

% An example of one such configuration is demonstrated on Fig.\AK{TODO if there will be more place}.

For experiments, 50 environments were generated. For every environment, trajectories of length $[5; 15; 30]$ were evaluated using the described pipeline.

Results of evaluation are presented on Fig.~\ref{fig:synth-eval}~(a). They demonstrate that usage of orthogonal subsets increases the quality of metric for every correlation coefficient, and orthogonal-based metrics achieve maximum value equal to 1 of the correlation coefficient. Also, high Pearson correlation demonstrates that MOM provides linear correlation with Relative Pose Error that gives more preference to its usage as no-reference map metric with strong dependence with trajectory error. This behavior could be explained by linear dependency between covariance along axis (which is actually a plane variance along plane normal) and RPE along this axis, demonstrated in Sec.~\ref{sec:lemma}. It could also be noticed, that all correlation coefficients grow with respect to number of poses in the trajectory. This behavior could be explained by the fact of obtaining a denser point cloud in the resulting map, that provides more information on point cloud consistency and errors in trajectory.

\subsubsection{CARLA simulator}

The CARLA simulator was used to generate outdoor LiDAR point clouds and corresponding sensor ground truth trajectories to evaluate metrics on environments with a richer structure. Compared to other simulators and datasets that support LiDAR generating (PreSIL~\cite{hurl2019precise},  Mai City Dataset), CARLA satisfies all the next requirements: 360 degrees of rotation, different outdoor environments, configurable types of LiDAR.

For evaluation, we consider 10 different scenes of 3 CARLA towns with 32-beam LiDAR and horizontal FoV in $[-10, 30]$ degrees, trajectory lengths are $[5; 15; 30]$.

Following main sampling pipeline we obtained the next results depicted in Fig.~\ref{fig:synth-eval}. For all trajectories the trend of orthogonal subset affection remains the same as in synthetic planes dataset~--- OMME and MOM outperform their ancestors, proving reliability of orthogonal-based approach in correlation with trajectory error. Among OMME and MOM, MOM provides better linear correlation, as demonstrated in Pearson statistics.

\subsection{Real environment}
To demonstrate that the algorithm for orthogonal subset extraction is suitable in real-life scenarios, especially outdoors, where ground truth poses are usually noisy for small-scale mapping, we evaluate it on urban KITTI maps (00, 05, 06, 07)~\cite{geiger2012we}. Correct orthogonal subset extraction was visually assessed on every 100 point cloud in those sequences. Also, statistics on the amount of extracted points in orthogonal subsets were calculated and is depicted in Fig.~\ref{fig:kitti-stat}. Statistics demonstrate that not more than 2k points are extracted from the original KITTI point cloud ($>$120k points), which means that MOM requires 60 times fewer points in comparison to its ancestors, speeding up metric calculations.

\begin{figure}[h]
    \centering
    \includegraphics[width=0.95\columnwidth]{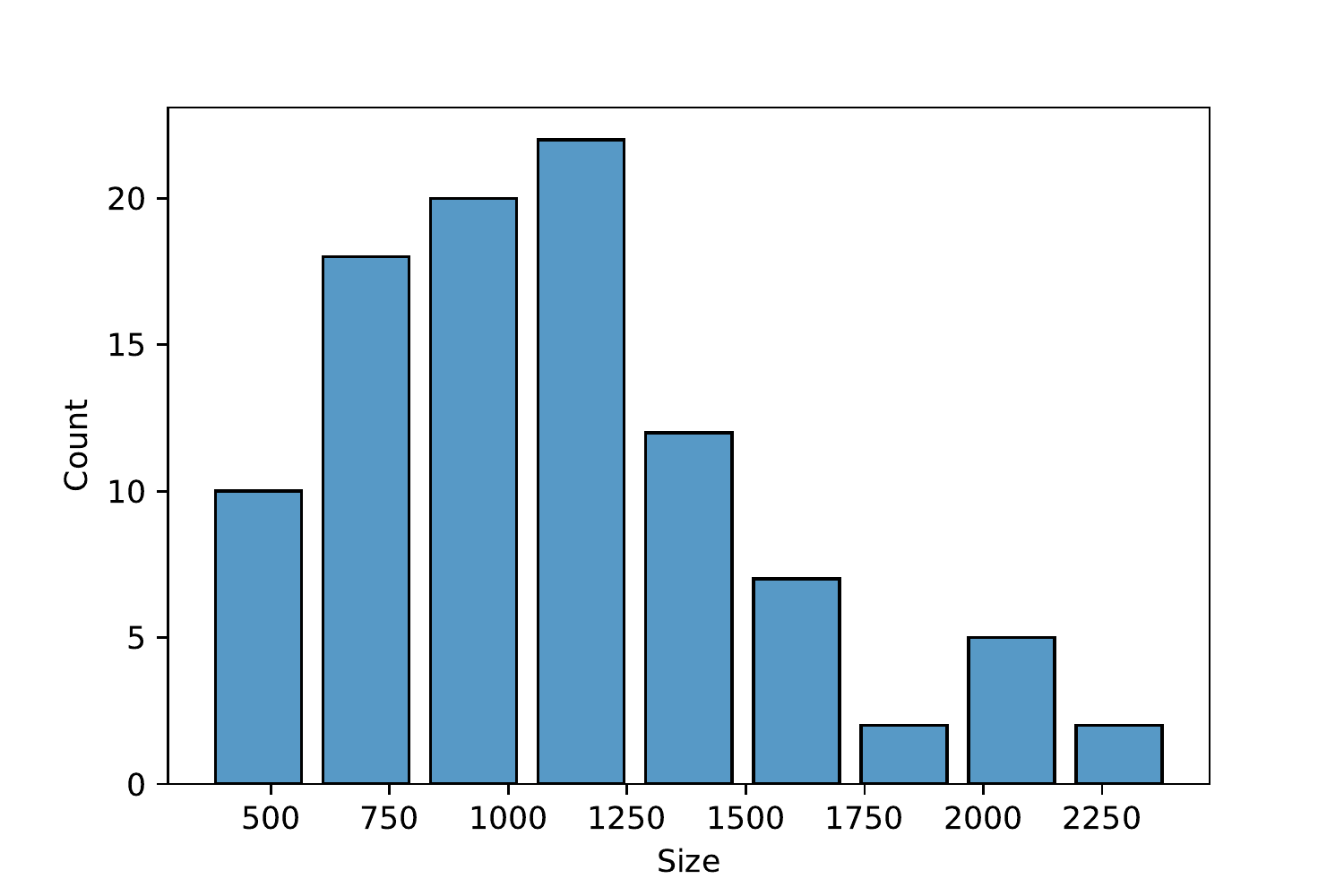}
    \caption{Statistics of extracted points from mutually orthogonal surfaces for every 100th point cloud in urban KITTI scenes.}
    \label{fig:kitti-stat}
\end{figure}

\subsection{Quality of KITTI odometry benchmark}

The proposed map metric could be used not only as an independent tool for trajectory and aggregated map quality assessment when GT is not available, but also as a tool for filtering out weak parts of benchmarks that have GT poses. Such an example is demonstrated in Fig.~\ref{fig:kitti-hotmap}, where MOM is calculated for every ten consecutive poses of the KITTI map providing information on how reliable are GT poses for different parts of the trajectory. Some parts of the map have low map consistency that means that GT poses are potentially not so accurate, that could lead to incorrect interpretation of odometry results on those parts of the trajectory.  

\begin{figure}[h]
    \centering
    \includegraphics[width=0.99\columnwidth]{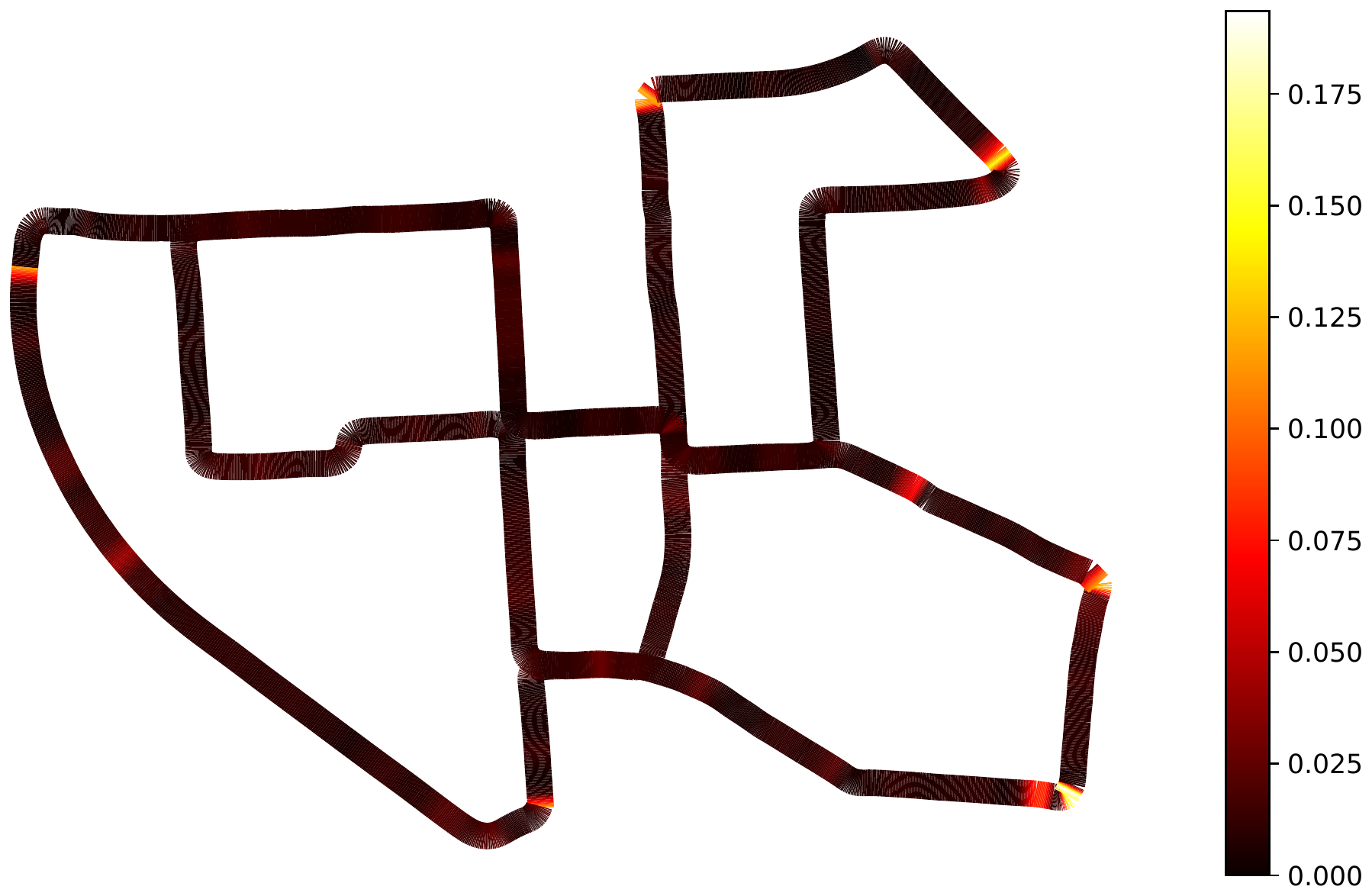}
    \caption{MOM map metric results for small-scale maps aggregated from every ten consecutive poses of KITTI odometry benchmark (map ``00''). Lower value (black) means more consistent map, higher value (yellow)~--- map with distortions caused by GNSS/INS imperfections.}
    \label{fig:kitti-hotmap}
\end{figure}

%% file: parts/6_conclusion.tex
\section{Conclusion}

This paper has presented a no-reference metric MOM for evaluating trajectory quality by assessing the map's quality from registered point clouds via trajectory poses. Our metric employs map points from mutually orthogonal surfaces, thereby providing a strong correlation with full-reference trajectory error in comparison to previous works, and requires 40-60 times fewer points of the map for evaluation. In our work, we have proposed a mathematical interpretation of this correlation under some assumptions. Also, we have conducted statistical experiments in two synthetic environments, including the CARLA simulator that offers close to real-life outdoor scenes and LiDAR point clouds collection. Finally, since our metric requires points from mutually orthogonal surfaces, we developed an algorithm for their extraction from point clouds and evaluated its' performance on KITTI urban scenes.